\documentclass[letterpaper]{article}

\usepackage[authoryear]{natbib}

\usepackage{geometry}
 \usepackage{times} 
 \usepackage{helvet} 
 \usepackage{courier}
\global\long\def\cip{\overset{p}{\rightarrow}}
\providecommand{\argmax}    {\operatornamewithlimits{arg\,max}}
\providecommand{\defas}     {\mathrel{\triangleq}}

\newcommand{\set}[1]{\mathcal{#1}}
\providecommand{\joint}[1]{\vec{#1}}

\providecommand{\excl}[1]{-{#1}}

\providecommand{\argsA}[2]{ {#1}_{#2} } %
\providecommand{\argsG}[2]{ {\joint{#1}}_{#2} } %
\providecommand{\argsI}[2]{ {#1}_{#2} } %

\providecommand{\hor}{h}

\providecommand{\AC}{a}                 %
\providecommand{\ACS}{\set{A}}                     %

\providecommand{\aAS}[1]{\argsA {\ACS}  {#1}}

\providecommand{\ja}    {\joint {\AC}}          %
\providecommand{\jaG}[1]    {\argsG{\AC}    {#1}}       %
\providecommand{\jaGS}[1]   {\argsG{\ACS}   {#1}}       %

\providecommand{\POL}{\pi}

\providecommand{\jpol}      {\joint{\POL}}	                %

\providecommand{\jpolG}[1]  {\argsG     {{\POL}}{#1}}	%

\providecommand{\Q}{Q}

\providecommand{\QI}[1]    {\argsI\Q{#1}}

\providecommand{\neig}[1]{\mathcal{N}({#1})} %

\usepackage{algpseudocode}
\usepackage{algorithm}

\usepackage{url}
\usepackage{subfigure}
\let\subfloat\subfigure
\usepackage{graphicx}
\providecommand{\E}{\mathbf{E}} %<- expectation
\providecommand{\hor}{H} %<- horizon

\usepackage[latin9]{inputenc}
\usepackage{calc}
\usepackage{amsmath}
\usepackage{amssymb}
\usepackage{tikz}

\usepackage{amsbsy}
\usepackage{amsthm}

\newcommand{\sect}[1]{Sec.~\ref{sec:#1}}
\newcommand{\eq}[1]{\eqref{eq:#1}}

\newcommand{\fig}[1]{Fig.~\ref{fig:#1}}

\newcommand{\alg}[1]{Algorithm~\ref{alg:#1}}

\providecommand{\newdef}[2]{\newtheorem{#1}{#2}}
\newtheorem{theorem}{Theorem}%[section]

\newtheorem{proposition}{Proposition}

\newdef{definition}{Definition} %[section]
\newdef{observation}{Observation} 
\newdef{conj}{Conjecture} %[section]

\providecommand{\ourpar}[1]{
\vspace{0.2em}
\noindent
\textbf{#1.}
}

\providecommand{\ja}    {\joint {\AC}}   
\providecommand{\jaG}[1]    {\argsG{\AC}    {#1}} 
\providecommand{\jaGS}[1]   {\argsG{\ACS}   {#1}} 
\providecommand{\argsG}[2]{ {\joint{#1}}_{#2} } 
\providecommand{\argsA}[2]{ {#1}_{#2} } 
\providecommand{\aAS}[1]{\argsA {\ACS}  {#1}}      
\providecommand{\AC}{a}    
\providecommand{\ACS}{\set{A}}            
\providecommand{\joint}[1]{\vec{#1}}
\global\long\def\cip{\overset{p}{\rightarrow}}
\providecommand{\jpol}      {\joint{\POL}}	
\providecommand{\POL}{\pi}
\providecommand{\defas}     {\mathrel{\triangleq}}
\providecommand{\QI}[1]    {\argsI\Q{#1}} 
\providecommand{\argsI}[2]{ {#1}_{#2} } %
\providecommand{\Q}{Q}
\providecommand{\neig}[1]{\mathcal{N}({#1})} %
\renewcommand{\cite}[1]{\citep{#1}}

\usepackage{verbatim}
\usepackage{float}
\usepackage{graphicx}
\usepackage{amsmath}
\usepackage{amssymb}
\usepackage{amsbsy}
\usepackage{amsthm}
\usepackage{colortbl}
\usepackage{algpseudocode}
\usepackage{algorithm}
\usepackage{psfrag}
\usepackage{longtable}
\usepackage{verbatim}
\usepackage{icomma} %intelligent comma in math
\providecommand{\ourpar}[1]{
\vspace{0.2em}
\noindent
\textbf{#1.}
}

\providecommand{\E}{\mathbf{E}} %<- expectation
\providecommand{\hor}{H} %<- horizon

\usepackage{subfigure}
\let\subfloat\subfigure
\def\<#1>{%
    \expandafter\ifx\csname<#1>\endcsname\relax
        \errmessage{abbreviation <#1> undefined!}
    \else
        \csname<#1>\endcsname
    \fi
}
\def\abbr#1#2{%
    \expandafter\def\csname<#1>\endcsname{#2}%
}

\tikzset{every node/.append style={scale=0.75}}

\begin{document}

\title{Scalable Planning and Learning for Multiagent POMDPs: Extended Version}

\author{
Christopher Amato\\ 
CSAIL, MIT\\ Cambridge, MA 02139 \\
\texttt{camato@csail.mit.edu} 
\and 
Frans A. Oliehoek\\
Informatics Institute, University of Amsterdam\\
Dept. of CS, University of Liverpool\\
\texttt{frans.oliehoek@liverpool.ac.uk}
}

\date{}

\maketitle

\begin{abstract}
Online, sample-based planning algorithms for POMDPs have shown great promise in scaling to problems with large state spaces, but they become intractable for large action and observation spaces. This is particularly problematic in multiagent POMDPs where the action and observation space grows exponentially with the number of agents. To combat this intractability, we propose a novel scalable approach based on sample-based planning and factored value functions that exploits structure present in many multiagent settings. This approach applies not only in the planning case, but also in the Bayesian reinforcement learning setting. Experimental results show that we are able to provide high quality solutions to large multiagent planning and learning problems. 
\end{abstract}

\section{Introduction}

Online planning methods for POMDPs have demonstrated impressive performance \cite{Ross08} on large problems by interleaving planning with action selection. The leading such method, partially observable Monte Carlo planning (POMCP)~\cite{Silver10NIPS23}, 
achieves performance gains by extending sample-based methods based on Monte Carlo tree search (MCTS) to solve POMDPs.

While online, sample-based methods show promise in solving POMDPs with large state spaces, they become intractable as the number of actions or observations grow. This is particularly problematic in the case of multiagent systems. Specifically, we consider multiagent partially observable Markov decision processes (MPOMDPs), which assume all agents share the same partially observable view of the world and can coordinate their actions. 
Because the MPOMDP model is centralized, POMDP methods apply, but the fact that the number of \emph{joint} actions and observations scales exponentially in the number of agents renders current POMDP methods intractable. 
%In theory, these methods also apply to learning in MPOMDPs, but current approaches quickly become intractable as the number of \emph{joint} actions and observations scales exponentially in the number of agents.

To combat this intractability, we provide a novel sample-based online planning algorithm that
exploits multiagent structure.
%we propose exploiting structure in the value functions associated with the agents. 
Our method, called \emph{factored-value partially observable Monte Carlo
planning (FV-POMCP)}, is based on 
%a Monte Carlo tree search (MCTS) variant called 
POMCP
%which has shown promise for planning in large POMDPs \cite{Silver10NIPS23} and Bayesian reinforcement learning (BRL) in large MDPs \cite{Guez12}. 
%
and is the first MCTS method that  exploits \emph{locality of interaction}: in many MASs,
agents interact directly with a subset of
other agents. This structure enables a decomposition of the value function into a
set of overlapping factors, which can be used to produce high quality solutions
\cite{Guestrin01,Nair05,Kok06JMLR}. 
But unlike these previous approaches, we will not assume a factored model, but only that the value function can be approximately factored. 
We present two variants of FV-POMCP that use different amounts of factorization
of the value function
to scale to large action and observation spaces.
%techniques for incorporating such factored value functions into
%MCTS,
%\todo{rephrase:}reducing the search for actions and observations at each step to potentially be linear.  
%thereby mitigating the additional challenges for scalability imposed by the
%exponential number of joint actions and observations.

Not only is our FV-POMCP approach applicable to large MPOMDPs, but it is potentially even
more important for Bayesian learning where the agents have uncertainty about the underlying model
as modeled by Bayes-Adaptive POMDPs (BA-POMDPs)~\cite{Ross11}. 
These models translate the learning problem to a planning problem, but since the resulting 
planning problems have an infinite number of states, scalable sample-based planning approaches are
critical to their solution.

We show experimentally that our approach allows both planning and learning to be significantly more efficient in 
multiagent POMDPs.
This evaluation shows that our approach significantly outperforms
regular (non-factored) POMCP, indicating that FV-POMCP is able to effectively exploit
locality of interaction in both settings.

\section{Background }
\label{background}
We first discuss multiagent POMDPs and previous work on Monte Carlo tree search and Bayesian reinforcement learning (BRL) for
 POMDPs. % \cite{Kaelbling98}. 

%%\vspace{-0.1cm}   
\subsection{Multiagent POMDPs}
An MPOMDP  \cite{Messias11} is a multiagent planning model that unfolds over a number of steps. 
%MPOMDPs form a framework for multiagent planning under uncertainty for a team of agents.
 At every stage, agents take individual actions and receive individual
observations. 
However, in an MPOMDP, all individual
observations are shared via communication, allowing the team of agents to act in a
`centralized manner'. We will restrict ourselves 
to the setting where such communication is free of noise, costs and delays. 

%Formally,
An MPOMDP is a tuple $ \langle  I, S, \{A_i\}, T, R, \{Z_i\}, O, \hor\rangle $
with: %\vspace{-10pt}
%\begin{itemize} \addtolength{\itemsep}{-1mm} 
%\item 
$I$, a set of agents;
%\item 
$S$, a set of states with designated initial state
distribution $b_0$; 
%\item 
 $A=\times_iA_i$, the set of joint actions, using action sets for each
agent, $i$;
%\item
  $T$, a set of state transition probabilities: 
$T^{s\vec{a}s'} = \Pr(s'| s, \vec{a})$, 
the
probability of transitioning from state $s$ to $s'$ when actions
$\vec{a}$ are taken by the agents; 
%\item  
$R$, a reward function: $R(s, \vec{a})$, the immediate reward for being in state $s$ and taking 
actions $\vec{a}$; 
%\item 
$Z=\times_iZ_i$, the set of joint observations, using observation sets for
each agent, $i$;
%\item  
$O$, a set of observation probabilities:
 $O^{\vec{a}s'\vec{z}} = \Pr(  \vec{z}|\vec{a},s')$, 
 the
probability of  seeing observations $\vec{o}$  given
actions $\vec{a}$  were taken and resulting state $s'$;
%\item 
$\hor$, the horizon.
%\end{itemize}

An MPOMDP can be reduced to a POMDP with a single
centralized controller that takes joint actions and receives joint
observations~\cite{Pynadath02}.
Therefore, MPOMDPs can be solved with POMDP solution methods, some 
of which will be described in the remainder of this section.
However, such approaches do not exploit the particular structure inherent to
many MASs. In \sect{FVPOMCP}, we present a first online planning method that
overcomes this deficiency. 
 
\subsection{Monte Carlo Tree Search for (M)POMDPs}
Most research for (mutliagent) POMDPs has focused on \emph{planning}:
given a full specification of the model, determine an optimal policy,
$\pi$, mapping past observation histories 
(which can be summarized by distributions $b(s)$ over states called beliefs) to 
actions.
%(e.g., \cite{Hansen04}) maximizing the value function for an initial state (or state distribution): %$$V^{\pi}(s)=R(s,\vec a)+ \sum_{s'}T^{s\vec{a}s'}\sum_{\vec z}O^{\vec{a}s'\vec{z}}V^{\pi}(s')$$
An optimal policy can be extracted from an optimal Q-value function,
$
Q(b,a) = \sum_s R(s,a) + \sum_z P(z|b,a) \max_{a'} Q(b',a')
$,
by acting greedily. %, in a way similar to regular MDPs. %~\cite{SuttonBarto98}. 
Computing $Q(b,a)$, however, is complicated by the fact that the space of beliefs is 
continuous.%~\cite{Kaelbling98}.

%A successful recent online planning method, 
%called %partially observable Monte Carlo planning (POMCP) 
POMCP~\cite{Silver10NIPS23}, is a scalable method which extends
Monte Carlo tree search (MCTS) %, and in particular the UCT algorithm \cite{Kocsis06}, 
to solve POMDPs.
At every stage, the algorithm performs online planning, given the current belief, by
incrementally building a lookahead tree that contains (statistics that represent)
$Q(b,a)$. 
The algorithm, however, avoids expensive belief updates by creating nodes not for each
belief, but simply for each action-observation history $h$. In particular, it samples
hidden states $s$ only at the root node (called `root sampling') and uses that state to
sample a trajectory that first traverses the lookahead tree and then performs a (random)
rollout. The return of this trajectory is used to update the statistics for all visited
nodes.
Because this search tree can be enormous, the search is directed to the relevant parts by selecting actions to maximize the `upper confidence
bounds': 
%\[
$
U(h,a) = Q(h, a) + c  \sqrt{\log(N + 1) / n}.
$
%\]
Here, $N$ is the number of times the history has been reached and $n$ is the number of times
that action $a$ has been taken in that history.  
%When the exploration constant $c$ is set correctly, 
POMCP
can be shown to
converge to an $\epsilon$-optimal value function. Moreover, the method has
demonstrated good performance in large domains with a limited number of simulations.
%For more details, we refer to  \cite{Silver10NIPS23}.

%\vspace{-0.25cm}
%%\vspace{-0.1cm}   
\subsection{Bayesian RL for (M)POMDPs}
% (multiagent)

\emph{Reinforcement learning (RL)}
considers the more realistic case where the model is not (perfectly) known in advance.
Unfortunately, effective RL in POMDPs is very difficult. \citet{Ross11} introduced a
framework, called the Bayes-Adaptive POMDP (BA-POMDP), that reduces the learning
problem to a planning problem, thus enabling advances in planning methods to 
be used in the learning problem. 

In particular, the BA-POMDP utilizes Dirichlet distributions to 
model uncertainty over transitions %, $T^{sas'}$, 
and observations. %, $O^{as'z}$.
Intuitively, if the agent could observe states and observations,
it could maintain vectors  $\phi$ and $\psi$ of counts for transitions and
observations respectively.  
Let $\phi^a_{ss'}$ be the transition count of the number times
state $s'$ resulted from taking action $a$ in state $s$ and $\psi^a_{s'z}$ be
the observation count representing the number of times observation $z$ was seen
after taking action $a$ and transitioning to state $s'$. 
These counts induce a %(product of Dirichlets)
 probability distribution over the possible
transition and observation models.
Even though the agent cannot observe the states
and has uncertainty about the actual count vectors, 
\emph{
this uncertainty can be represented
using the POMDP formalism} --- by including the count vectors as part of the hidden state of a special
POMDP, called a BA-POMDP.
%That is, the count vectors are included as part of the hidden state of a special
%POMDP, called BA-POMDP.

The BA-POMDP can be extended to the multiagent setting \cite{Amato13MSDM}, yielding
the  Bayes-Adaptive multiagent POMDP (BA-MPOMDP) framework.
BA-MPOMDPs are POMDPs, but with an infinite state space since there can be infinitely many
count vectors.  While a quality-bounded reduction to a finite state
space is possible~\cite{Ross11}, the problem is still intractable; sample-based planning
is needed to provide solutions. Unfortunately, current methods, such as POMCP, do not scale well
to multiple agents.

%\vspace{-0.15cm}   
\section{Exploiting Graphical Structure}

POMCP %(and MCTS) are 
is not directly suitable for multiagent problems (in either the planning
or learning setting) due to the fact that the number of joint actions and observations are
exponential in the number of agents.  
We first elaborate on these problems, and then sketch an approach to
mitigate them by exploiting locality between agents.

%%\vspace{-0.1cm}   
\subsection{POMCP for MPOMDPs: Bottlenecks}

The large number of joint observations is problematic
since it leads to a lookahead tree with very high branching factor.
Even though this is theoretically not a problem in MDPs \cite{Kearns02ML},
in partially observable settings 
that use particle filters it leads to severe problems.
In particular, in order to have a good particle representation at the next
time step, the actual joint observation received
must be sampled often enough during planning for the previous stage.
If the actual joint observation had not been sampled frequently enough 
(or not at all), the particle filter will be a bad approximation (or collapse).
This results in sampling starting from the initial belief again, or
alternatively, to fall back to acting using a separate (history independent) policy such as a random one.

%joint actions
The issue of large numbers of joint actions is also problematic: the
standard POMCP %UCT 
algorithm will, at each node, maintain separate statistics, and thus separate
upper confidence bounds, for each of the exponentially many joint actions. %This means that
Each of the exponentially many joint actions will have to be selected at least a few
times to reduce their confidence bounds (i.e., exploration bonus). 
This is a principled problem:
%Clearly this is intractable for all but the smallest problems.
%The problem of joint actions is a principled one: 
in cases where
each combination of individual actions may lead to completely different effects,
it may be necessary to 
%there may be no better thing to do than
 try all of them at least a few
times. In many cases, however, the effect of a joint action is factorizable as the
effects of the action of individual agents or small groups of agents. For
instance, consider a team of agents that is fighting fires at a
number of burning houses, as illustrated in \fig{fire}. 
%In this setting, 
%The transition probabilities can be factored as a product of transition probabilities for each house and the transitions of a particular house 
The rewards depend only on the amount of water deposited on each house
rather than the exact joint action taken \cite{Oliehoek08AAMAS}. While this problem lends itself to a natural factorization, other problems may also be factorized to permit approximation. 

\begin{figure}[tb]
\centering
\subfloat[]{\label{fig:fire}\includegraphics[width=40mm]{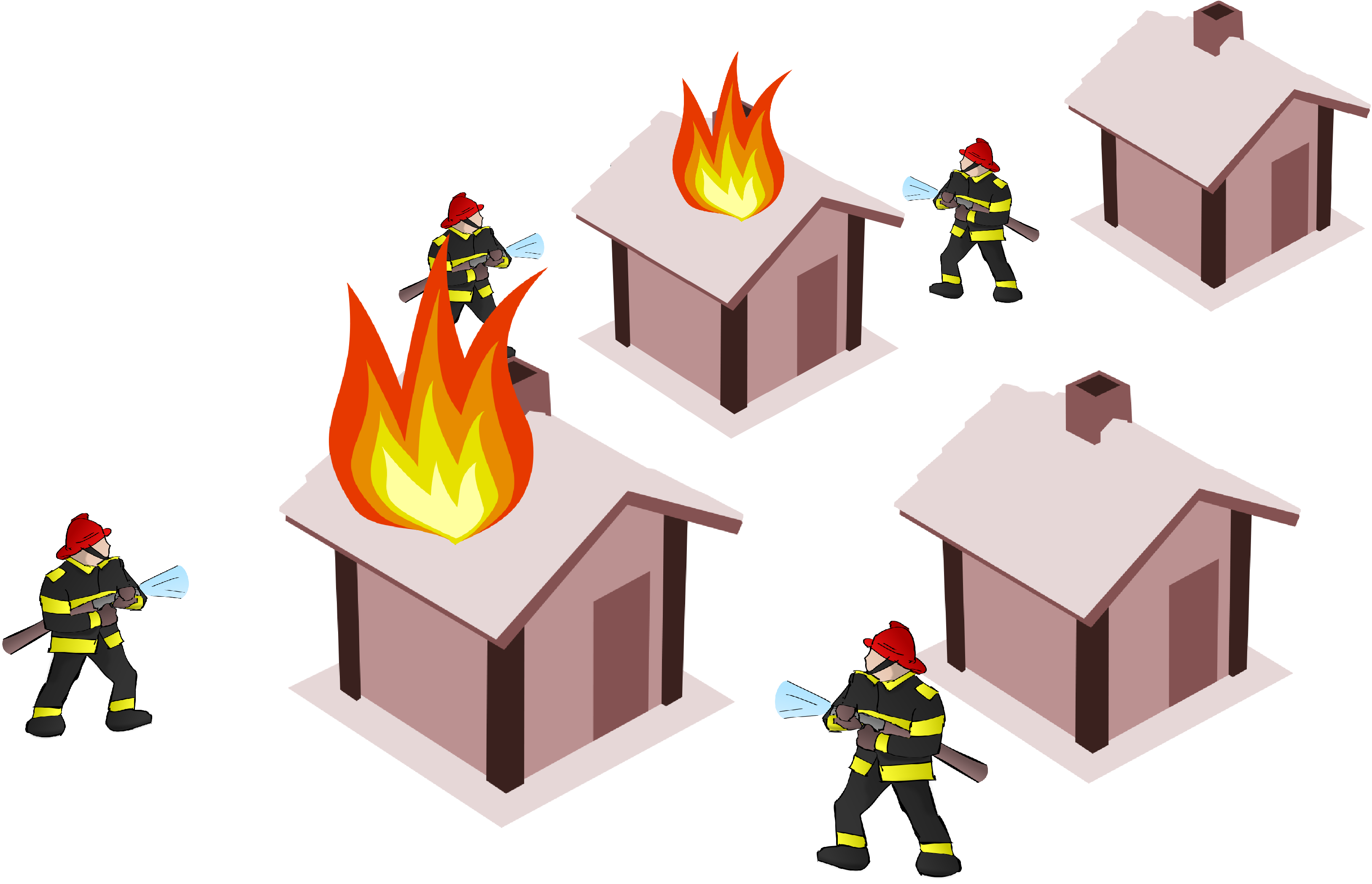}} \hfill
 \subfloat[]{\label{fig:fire-graph}\includegraphics[width=40mm]{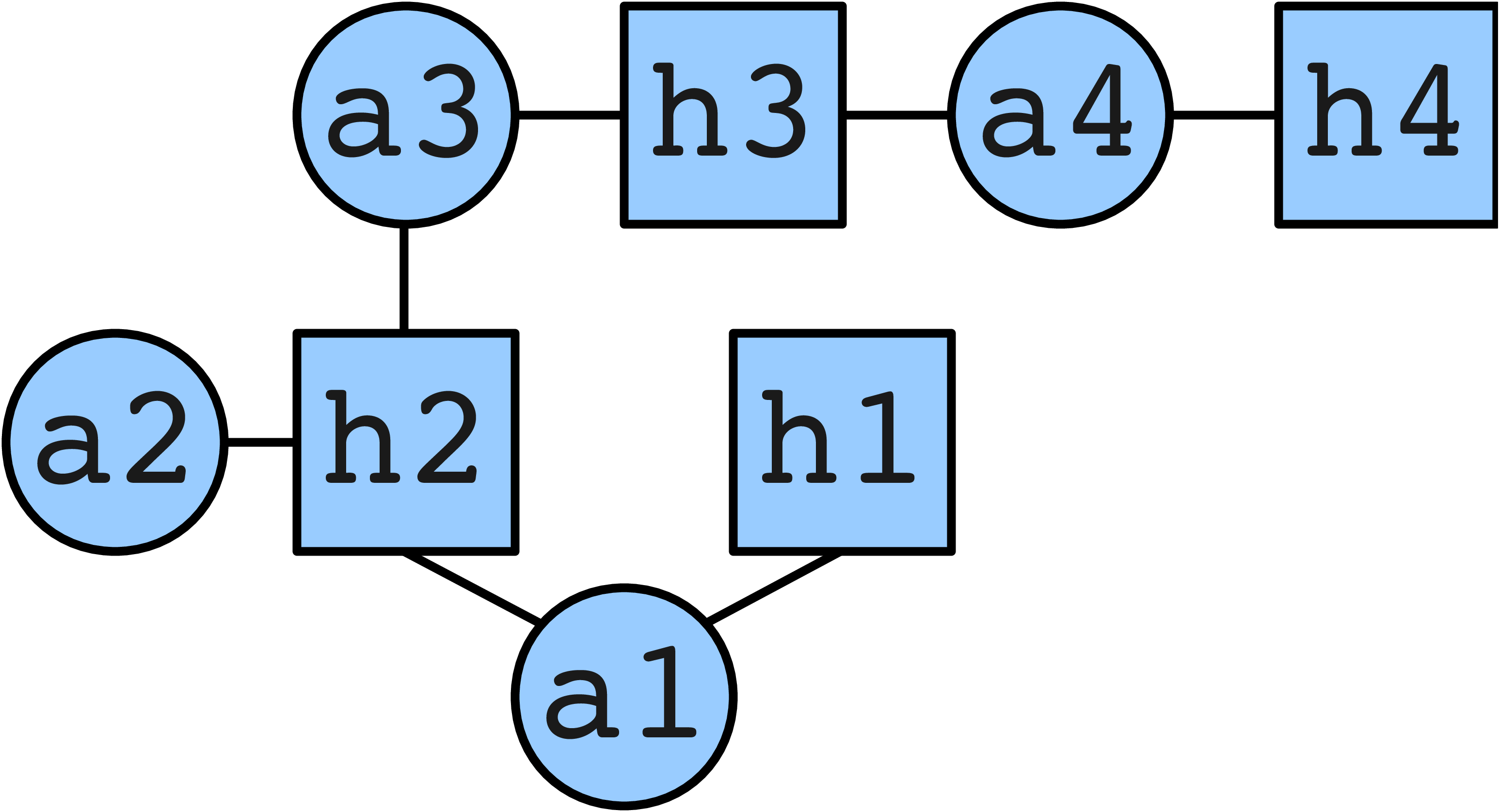}} \hfill
 \subfloat[]{\label{fig:sensor-grid}\includegraphics[width=35mm]{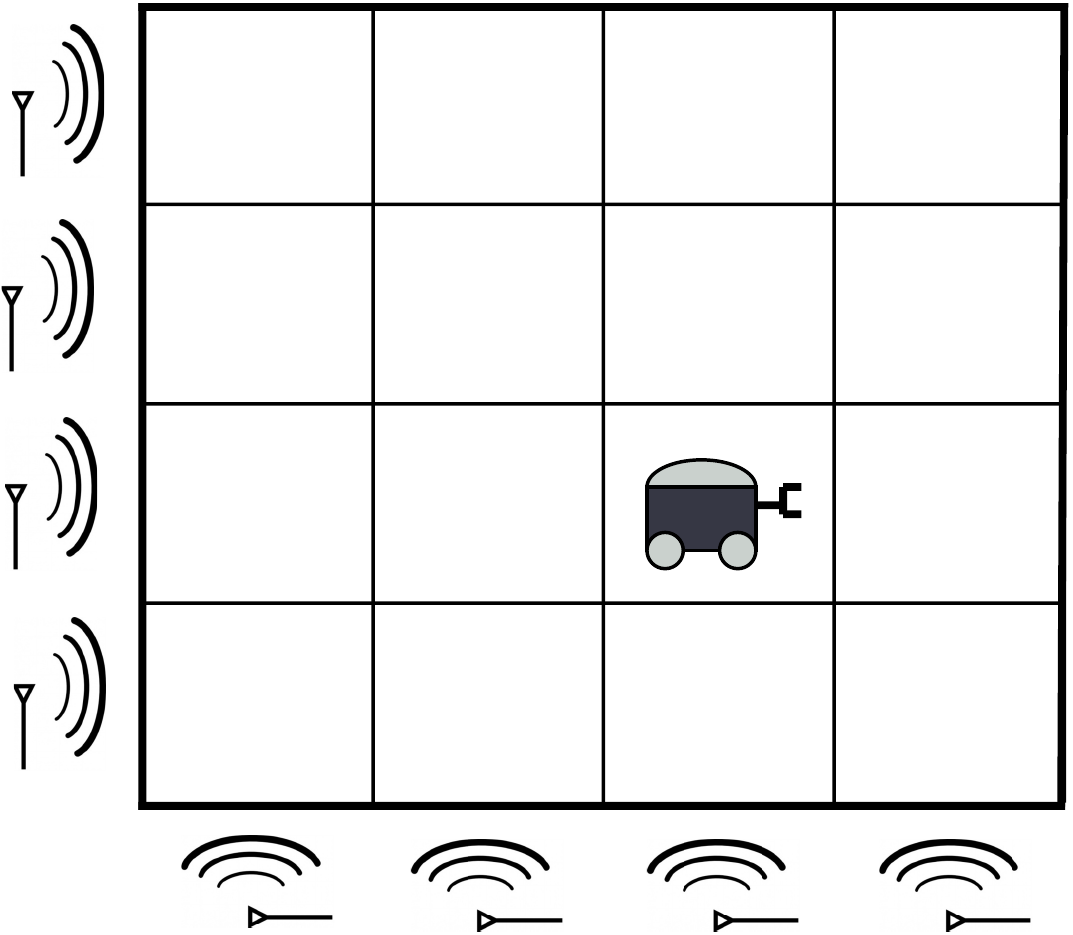}}
% \vspace{-5pt}
\caption{
(a) Illustration of `fire fighting'
(b) 
Coordination graph 
with 4 houses 
and 3 agents
(c)  Illustration of a sensor network problem on a grid that is used in the experiments.
}
\end{figure}

%\section{Exploiting Graphical Structure in Sample-Based Optimization}
%\vspace{-0.1cm}   
\subsection{Coordination Graphs}

In certain multiagent settings, \emph{coordination (hyper) graphs (CGs)}~\cite{Guestrin01,Nair05} 
have been used to compactly represents interactions between subsets of agents. In this
paper we extend this approach to MPOMDPs. We first introduce the framework of CGs in the
single shot setting.

A CG specifies a set of payoff components $E= \{Q_e\} $, and each component $e$ is associated
with a subset of agents. These subsets (which we also denote using $e$) can be interpreted 
as (hyper)-edges in a graph where the nodes are agents. The goal in a CG is to select a
joint action that maximizes the sum of the local payoff components $Q(\ja)=\sum_e
Q_e(\vec{a}_e)$.\footnote{%promised clarification:
    Since we focus on the one-shot setting here, the Q-values in the remainder of this section should
    be interpreted as those for one specific joint history~$h$, i.e.: $Q(\ja) \equiv Q(h,\ja)$.
}
A CG for the fire fighting problem is shown in \fig{fire-graph}.
We follow the cited literature in assuming that a suitable factorization is
easily identifiable by the designer, but it may also be learnable. 
Even if a payoff function $Q(\ja)$ does not factor exactly, it can be approximated by a
CG.  For the moment assuming a \emph{stateless problem} (we will consider the case where histories are included in the next section), an action-value function can be
approximated by 
\begin{equation}
Q(\vec a) \approx \sum_e Q_e({\vec a}_e), 
\label{eq:linear-approximation}
\end{equation}
%where each component $e$ is specified over only a subset of agents (leading to an
%`interaction hypergraph' \cite{Nair05}). 
We refer to this as the linear approximation 
of $Q$, since one can show that this corresponds to an instance of linear regression 
(See \sect{theory}).

Using a factored representation, 
%Rewards can be broken into local rewards over sets of agents, with agents that participate in the same local reward connected in the graph.  M
the maximization $\max_{\vec{a}} \sum_e Q_e(\vec{a}_e) $ can be performed efficiently via
variable elimination (VE) \cite{Guestrin01}, or max-sum \cite{Kok06JMLR}.
% took this out for now Farinelli08AAMAS
These algorithms are not exponential in the number of agents, and therefore
enable significant speed-ups for larger number of agents.
The VE algorithm (which we use in the experiments) is exponential in the \emph{induced
width} $w$ of the coordination graph. 

%\vspace{-0.1cm}   
\subsection{Mixture of Experts Optimization}

VE can be applied if the CG is given in advance. When we try to exploit these techniques
in the context of POMCP, however, this is not the case. 
As such, the task we consider here is to
find the maximum of an \emph{estimated}  factored function 
$\hat{Q}(\vec a) = \sum_e \hat{Q}_e(\vec{a}_e)$.
Note that we do not necessarily require the best approximation to the entire $Q$,
as in \eq{linear-approximation}. 
Instead, we seek an estimation $\hat{Q}$ for which the maximizing joint action $\vec{a}^M$
is close to the maximum of the actual (but unknown) Q-value: $Q(\vec{a}^M)
\approx Q(\vec{a}^*)$.

For this purpose, we introduce a technique called \emph{mixture of experts optimization}.
In contrast to methods based on linear approximation \eq{linear-approximation}, we do not
try to learn a best-fit factored Q function, but directly try to estimate the maximizing
joint action. 
The main idea is that for each local action $\vec{a}_e$ we introduce an expert that
predicts the \emph{total} value $\hat{Q}(\vec{a}_e) = \E[ Q(\vec{a}) \mid \vec{a}_e ] $. 
For a joint action, these responses---one of each payoff component $e$---are put in a
mixture with weights $\alpha_e$ and used to predict the maximization joint action:
$
\arg
\max_{\vec{a}}
\sum_{e}  \alpha_{e}\hat{Q}(\vec{a}_e).
$
This equation is the sum of restricted-scope functions, which is identical to the case of linear approximation  
\eq{linear-approximation},  so VE can
be used to perform this maximization effectively. In the remainder of this paper, 
we will integrate the weights and simply write $\hat{Q}_e(\vec{a}_e) =  \alpha_{e}\hat{Q}(\vec{a}_e)$.

The experts themselves are implemented as maximum-likelihood estimators of the total
value. That is, each expert (associated with a particular $\vec{a}_e$) keeps track of the
mean payoff received when $\vec{a}_e$ was performed, which can be done very efficiently.
An additional benefit of this approach is that it allows for efficient estimation of upper
confidence bounds by also keeping track of how often this local action was performed,
which in turns facilitates easy integration in POMCP, as we describe next.

%\vspace{-0.15cm}   
\section{Factored-Value POMCP}
\label{sec:FVPOMCP}

This section presents our main algorithmic contribution: Factored-Value POMCP, which is an
online planning method for POMDPs that can exploit approximate structure in the value
function by applying mixture of experts optimization in the POMCP lookahead search tree.
We introduce two variants of FV-POMCP. % based on two techniques to overcome exponentially large action and observation spaces.
The first technique, \emph{factored statistics}, only addresses the
complexity introduced by joint actions. The second technique, \emph{factored
trees}, additionally addresses the problem of many joint observations.
FV-POMCP is the first MCTS method to exploit structure in MASs, 
achieving better sample complexity by using factorization to generalize the value function
over joint actions and histories.    
While this method was developed to 
scale POMCP to larger MPOMDPs in terms of number of
agents, the techniques may  be beneficial in other multiagent models and other factored POMDPs.

%\vspace{-0.1cm}   
\subsection{Factored Statistics}
We first introduce \emph{Factored Statistics} which directly applies mixture of experts
optimization to each node in the POMCP search tree.  As shown in \fig{FS}, the tree of joint histories remains the same, but the statistics retained at for each history is now different. That is, rather than maintaining
one set of statistics in each node (i.e, joint history $\vec h$) for the expected value of each joint action $Q(\vec{h},
\vec{a})$, we maintain a set of statistic for each component $e$ that 
estimates the values $Q_e(\vec{h}, \vec{a}_e)$ and corresponding upper confidence bounds.

Joint actions are selected 
according to the maximum (factored) upper confidence bound:
%When selecting an action, we maximize over the sum of the factors, 
$$
\max_{\vec a}\sum_{e} U_e(\vec h, \vec a_e),
$$
%which follows directly from the factored form.
Where 
$
U_e(\vec h, \vec a_e)
\triangleq  
Q_e(\vec h, \vec a_e)
+ c  \sqrt{\log(N_{\vec h} + 1) / n_{\vec a_e}}
$
using the Q-value and the exploration bonus added for that factor.  
For implementation, at each node for a joint history $\vec{h}$, we store the count for the
full history $N_{\vec h}$ as well as the Q-values, $Q_e$, and the counts for actions,
$n_{\vec a_e}$, separately for each component $e$.

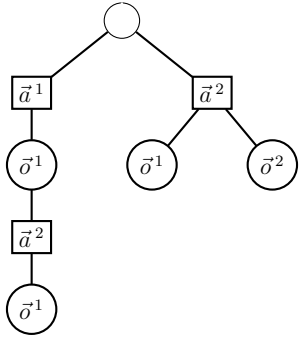
\begin{figure}[t]
\centering
\begin{tikzpicture}[level/.style={sibling distance=80mm/#1}, scale=0.6]
\node[](nulltmp){}[edge from parent/.style={draw,white,thick},level distance=1mm]
  child{
	node[circle,draw](root1){\quad\quad}
		child[edge from parent/.style={draw,black,thick},level distance=16mm, black]{node[draw] {\large $\vec a\,^1$}
			child[edge from parent/.style={draw,black,thick},level distance=16mm, black]{node[draw,circle]  {\large $\vec o\,^1$}
			%child{node[circle,draw] {\large $\vec o\,^2$}}  
			child[edge from parent/.style={draw,black,thick},level distance=16mm, black]{node[draw] {\large $\vec a\,^2$}
			child{node[circle,draw] {\large $\vec o\,^1$}}
			%child{node[circle,draw] {\large $\vec o\,^2$}}  
			}
			}
		}
		child[edge from parent/.style={draw,black,thick},level distance=16mm, black]{node[draw] {\large $\vec a\,^2$}
			child{node[circle,draw] {\large $\vec o\,^1$}}
			child{node[circle,draw] {\large $\vec o\,^2$}}  
		}
  };
\end{tikzpicture}
%%\vspace{-5pt}
\caption{Factored Statistics: joint histories are maintained (for specific joint actions and observations specified by $\vec a\,^j$ and $\vec o\,^k$), but action statistics are factored at each node. }
\label{fig:FS} 
\end{figure}

Since this method retains fewer statistics and performs joint action selection more
efficiently via VE, we expect that it will be more efficient than plain application of
POMCP to the BA-MPOMDP.  However, the complexity due to joint observations is not directly
addressed: because joint histories are used, reuse of nodes and creation of new nodes for
all possible histories (including the one that will be realized) may be limited if the
number of joint observations is large.

\subsection{Factored Trees}
The second technique, called \emph{Factored Trees}, additionally tries to overcome the
large number of joint observations.  It further decomposes the local $Q_e$'s by splitting
joint histories into local histories and distributing them over the factors.  
That is, in this case, we introduce an expert for each local $\vec{h}_e, \vec{a}_e$ pair.
During simulations, the agents know $\vec{h}$ and action selection is conducted by
maximizing over the sum of the upper confidence bounds:
$$
\max_{\vec a}\sum_{e} U_e(\vec h_e, \vec a_e),
$$
where 
$
U_e(\vec h_e, \vec a_e)= Q_e(\vec h_e, \vec a_e)+ c  \sqrt{\log(N_{\vec h_e} + 1)
/ n_{\vec{a}_e}}
$.
We assume that the set of agents with relevant actions and histories for component $Q_e$ are the same, but this can be generalized.
This approach further reduces the number of statistics maintained and increases
the reuse of nodes in MCTS and the chance that nodes in the trees will exist for 
observations that are seen during execution. 
As such, it can increase performance by increasing generalization as well as producing more robust particle filters.

\begin{figure}[t]
\centering
\begin{tikzpicture}[level/.style={sibling distance=60mm/#1}, scale=0.6]
\node[](nulltmp){ \Large $\ldots \quad e \quad \ldots$}[edge from parent/.style={draw,white,thick},level distance=1mm]
  child{
	node[circle,draw](root1){\quad\quad}
		child[edge from parent/.style={draw,black,thick},level distance=16mm, black]{node[draw] {\large $\vec a\,^1_1$}
			child[edge from parent/.style={draw,black,thick},level distance=16mm, black]{node[draw,circle]  {\large $\vec o\,^1_1$}
			%child{node[circle,draw] {\large $\vec o\,^2$}}  
			child[edge from parent/.style={draw,black,thick},level distance=16mm, black]{node[draw] {\large $\vec a\,^2_1$}
				child{node[circle,draw] {\large $\vec o\,^1_1$}}
			%child{node[circle,draw] {\large $\vec o\,^2$}}  
				}
			}
		}
		child[edge from parent/.style={draw,black,thick},level distance=16mm, black]{node[draw] { $\vec a\,^2_1$}
			child{node[circle,draw] {$\vec o\,^1_1$}}
			child{node[circle,draw] { $\vec o\,^2_1$}}  
		}
	}
  child{
	node[circle,draw](root1){\quad\quad}
		child[edge from parent/.style={draw,black,thick},level distance=16mm, black]{node[draw] { $\vec a\,^1_{|E|}$}
			child{node[circle,draw] { $\vec o\,^1_{|E|}$}}
			%child{node[circle,draw] { $\vec o\,^2_{|E|}$}}
		}
		child[edge from parent/.style={draw,black,thick},level distance=16mm, black]{node[draw] {$\vec a\,^2_{|E|}$}
			child[edge from parent/.style={draw,black,thick},level distance=16mm, black]{node[draw,circle]  {\large $\vec o\,^1_{|E|}$}
			%child{node[circle,draw] {\large $\vec o\,^2$}}  
			child[edge from parent/.style={draw,black,thick},level distance=16mm, black]{node[draw] {\large $\vec a\,^2_{|E|}$}
				child{node[circle,draw] {\large $\vec o\,^1_{|E|}$}}
			%child{node[circle,draw] {\large $\vec o\,^2$}}  
				}
			}
			child{node[circle,draw] { $\vec o\,^2_{|E|}$}}  
		}
  };
\end{tikzpicture}
%%\vspace{-5pt}
\caption{Factored Trees: local histories for are maintained for each factor (resulting in factored history and action statistics). Actions and observations for component $i$ are represented as $\vec a\,^j_i$ and $\vec o\,^k_i$) }
\label{fig:FT}
\end{figure}
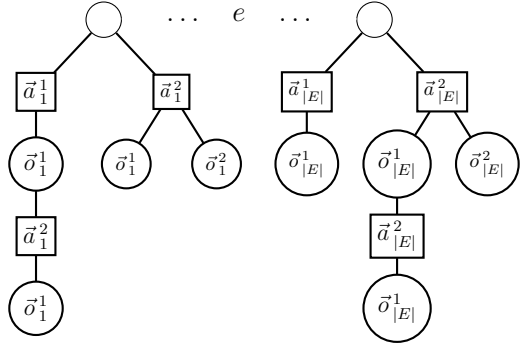

This type of factorization has a major effect on the implementation: rather than constructing a single tree, we now construct a number of
trees in parallel, one for each factor %(or edge in the  coordination graph)
$e$ as shown in \fig{FT}. A node of the tree for component $e$ now stores the required statistics:
$N_{\vec h_e}$, the count for the local history, $n_{\vec{a}_e}$, the counts
for actions taken in the local tree and $Q_e$ for the tree. 
Finally, we point out that this decentralization of statistics 
has the potential to reduce communication since the components statistics in a decentralized fashion could be updated without knowledge of all observation histories. 

%\vspace{-0.15cm}   
\section{Theoretical Analysis}
\label{sec:theory}

Here, we investigate the approximation quality induced by our factorization
techniques.\footnote{Proofs can be found in Appendix \ref{proofs}.} The
most desirable quality bounds would express the performance relative to `optimal', i.e.,
relative to  flat POMCP, which converges in probability 
 an $\epsilon$-optimal value function.
Even for the one-shot case, this is extremely difficult for any method employing factorization based on linear approximation of Q,  %due to the
%following observation:
%\begin{observation}
because 
    Equation \eq{linear-approximation} corresponds to a special case of linear regression. 
In this case, we can write \eq{linear-approximation} 
        in terms of basis functions and weights as:
$
    \sum_{e} Q_e(\vec{a}_e)
=   
    \sum_{e, \vec{a}_e} w_{e, \vec{a}_e} h_{e, \vec{a}_e} ( \vec{a}), 
$
where the $h_{e, \vec{a}_e}$ are the basis functions: 
$ h_{e, \vec{a}_e} ( \vec{a}) = 1$ iff $\vec{a}$ specifies $\vec{a}_e$ for component $e$ 
(and 0 otherwise).
%\end{observation}
As such, providing guarantees with respect to the optimal $Q(\vec{a})$-value would
require developing a priori bounds for the approximation quality of (a particular type of)
basis functions. This is a very difficult problem for which there is no good solution,
even though these methods are widely studied in machine learning. 

However, we do not expect our methods to perform well on arbitrary $Q$. Instead, we expect
them to perform well when $Q$ is nearly factored, such that \eq{linear-approximation}
approximately holds, since then the local actions contain enough information to make
good predictions. As such, we analyze the behavior of our methods  when the
samples of $Q$ come from a factored function (i.e., as in \eq{linear-approximation}
) contaminated with zero-mean noise. In such cases, we can show the following.
\begin{theorem}The estimate $\hat{Q}$ of $Q$ made by a mixture
of experts converges in probability to the true value plus a sample
policy dependent bias term:
$
\hat{Q}(\ja)\cip Q(\ja)+B_{\jpol}(\ja).
$
The bias is given by a sum of biases induced by pairs $e,e'$:
\[
B_{\jpol}(\ja)\defas\sum_{e}\sum_{e'\neq e}\sum_{\overline{\jaG{e'\setminus e}}}\jpol(\overline{\jaG{e'\setminus e}}|\jaG e)\QI{e'}(\overline{\jaG{e'\setminus e}},\jaG{e'\cap e}).
\]
Here, $\jaG{e'\cap e}$ is the action of the agents that participate
both in $e$ and $e'$ (specified by $\ja$) and $\overline{\jaG{e'\setminus e}}$
are the actions of agents in $e'$ that are not in $e$.
\end{theorem}

Because we observe the global reward for a given set of actions, the bias is caused by correlations in the sampling policy
and the fact that we are overcounting value from other components. When there is no
overlap, and the sampling policy we use is `component-wise': 
$\jpol({\jaG{e'\setminus e}}|\jaG e)=\jpol({\jaG{e'\setminus e}}|\jaG
e')=\jpol({\jaG{e'\setminus e}})$,
this over counting is the same for all local actions $\jaG e$:

\begin{theorem}When value components do not overlap and a
    component-wise sampling policy is used, mixture of experts optimization recovers the
    maximizing joint action.
\end{theorem}

Similar reasoning can be used to establish bounds on the performance in the case of
overlapping components, subject to assumptions about properties of the true value
function. 
Let $\neig e$ denote the neighborhood of component $e$: the set of other components $e'$ which
have an overlap with $e$ (those that have at least one agent participating
in them that also participates in $e$). 

\begin{theorem} If for all overlapping components $e,e'$, and any
two `intersection action profiles' $\jaG{e'\cap e},\jaG{e'\cap e}'$
for their intersection, the true value function satisfies 
\begin{equation}
\forall\jaG{e'\setminus e}\qquad\QI{e'}(\jaG{e'\setminus e},\jaG{e'\cap e})-\QI{e'}(\jaG{e'\setminus e},\jaG{e'\cap e}')
\leq
\frac
{\epsilon}
{
|E| 
\cdot |\neig e|
\cdot |\jaGS{e'\setminus e}|
\cdot \jpol(\jaG{e'\setminus e})
},
\end{equation}
%\end{multline}
with $|\jaGS{e'\setminus e}|$ the number of intersection action profiles,
then mixture of experts optimization, in the limit,
will return a joint action whose value lies within $\epsilon$
of the optimal solution. \end{theorem}

The analysis shows that  
a sufficiently local Q-function can be effectively
optimized when using a sufficiently local sampling policy. Under the same assumptions, we can also derive guarantees for the sequential case.
 It is not directly possible to derive bounds for FV-POMCP itself (since it is not possible to demonstrate
that the UCB exploration policy is component-wise), but it seems likely
that UCB exploration leads to an effective policy that nearly satisfies this property. 
Moreover, since bias is introduced by the interaction between action correlations and
differences in `non-local' components, even when using a policy with correlations, the
bias may be limited if the Q-function is sufficiently structured.

In the factored tree case, we can introduce a strong result. 
%Here we introduce a strong result on the behavior of FT-FV-POMCP.
Because histories for other agents outside the factor are not included and we do not assume independence between factors, 
the %solution 
approximation quality may suffer: where $\vec h$ is Markov,  
this is not the case for the local history
$\vec h_e$. 
As such, the expected return for such a local history depends on the future
policy as well as the past one (via the distribution over histories of agents not included
in $e$). 
This implies that convergence is no longer guaranteed:
\begin{proposition} 
    Factored-Trees FV-POMCP may diverge. 
\end{proposition}
%%\vspace{-2mm}
\begin{proof}%[Rationale]
        FT-FV-POMCP (with $c=0$) corresponds to a general case of Monte Carlo control (i.e., SARSA(1)) with linear function approximation that is greedy w.r.t. the current value function. Such settings may result in divergence \cite{Fairbank12IJCNN}.
\end{proof}
Even though this is a negative result, and there is no guarantee of 
convergence for FT-FV-POMCP, in practice this need not be a problem;
%First, we point out that this proof relies on setting the exploration constant $c=0$. Second, 
many reinforcement learning techniques that can diverge 
(e.g., neural networks) 
can produce high-quality results in practice, e.g., \cite{Tesauro95CACM,Stone01ICML}. 
Therefore, we expect that if the problem exhibits enough locality, the factored
trees approximation may allow good quality policies to be found very quickly. 

Finally, we analyze the computational complexity. % of FV-POMCP.
FV-POMCP is implemented by modifying POMCP's \textsc{Simulate} function %\cite{Silver10NIPS23}
(as described in Appendix \ref{code}).
The maximization is performed by variable elimination, which has complexity 
$O(n |\aAS{max}|^w)$ with $w$ the induced width and $|\aAS{max}|$ the size of the largest
action set. In addition, the algorithm updates each of the $|E|$ components,
bringing the total complexity of one call of simulate to 
$O(|E| + n |\aAS{max}|^w)$.

%\vspace{-0.15cm}   
\section{Experimental Results}

Here, we empirically investigate the effectiveness of our factorization methods by comparing them to non-factored methods in the planning and learning settings.

\ourpar{Experimental Setup}
%\todo{what else here?}
We test our methods on versions of the firefighting problem from
Section~\ref{sec:FVPOMCP} and on sensor network problems.
In the firefighting problems, fires are suppressed more quickly if a larger number of agents choose that
particular house.  Fires also spread to neighbor's houses and can start at any
house with a small probability.  In the sensor network problems (as shown by \fig{sensor-grid}), sensors were aligned along discrete intervals on two axes with rewards for tracking a target that moves in a grid. Two types of sensing could be employed by each agent (one more powerful than the other, but using more energy) or the agent could do nothing. A higher reward was given for two agents correctly sensing a target at the same time. 
The firefighting problems were broken up into $n-1$ overlapping factors with 2 agents in each (representing the agents on the two sides of a house) and the sensor grid problems were broken into $n/2$ factors with $n/2+1$ agents in each (representing all agents along the y axis and one agent along the x axis).  
For the firefighting problem with 4 agents, $|S|$ = 243, $|A|$ = 81 and $|Z|$ = 16
 and with
10 agents,  $|S|$ = 177147, $|A|$ = 59049 and $|Z|$ = 1024. For the sensor network problems with 4 agents, $|S|$ = 4, $|A|$ = 81 and $|Z|$ = 16
and with
8 agents, $|S|$ = 16, $|A|$ = 6561 and $|Z|$ = 256.

Each experiment was run for a given number of \emph{simulations}, the number of samples
used at each step to choose an action, and averaged over a number of
\emph{episodes}. 
We report undiscounted return with the standard
error.
Experiments were run on a single core of a 2.5 GHz machine with 8GB of memory. 
In both cases, we compare our factored representations to the flat version using POMCP. This comparison uses the same code base so it directly shows the difference when using factorization. POMCP and similar sample-based planning methods have already been shown to be state-of-the-art methods in both POMDP planning \cite{Silver10NIPS23} and learning \cite{Ross11}.

%>>>>>>>>>>>---------------------------------------------------------------------------------
\ourpar{MPOMDPs}
We start by comparing the factored statistics (FS) and factored
tree (FT) versions of FV-POMCP in multiagent planning problems. Here, the agents are given the true MPOMDP model (in the form of a simulator)
and use it to plan.
%>>>>baselines
For this setting, we compare to two baseline methods:
\emph{POMCP}:  regular POMCP applied to the MPOMDP, and
\emph{random}: uniform random action selection.
Note that 
while POMCP will converge to an $\epsilon$-optional solution, the solution quality may be poor when using small number of simulations. 

\begin{figure}[bt]
\centering
%\hspace{-8pt}
\subfigure[MPOMDP results]{
\label{fig:FFGtrue}
\hspace{-5pt}
\includegraphics[width=2.05in]{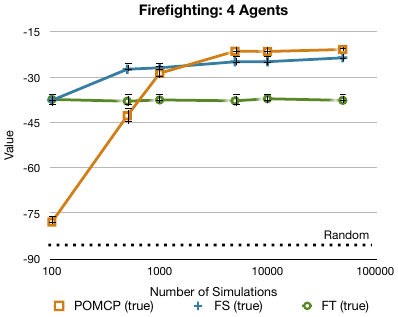}\hspace{-5pt}
\includegraphics[width=2.1in]{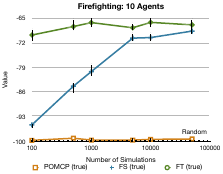}\hspace{-5pt}
 \includegraphics[width=2.1in]{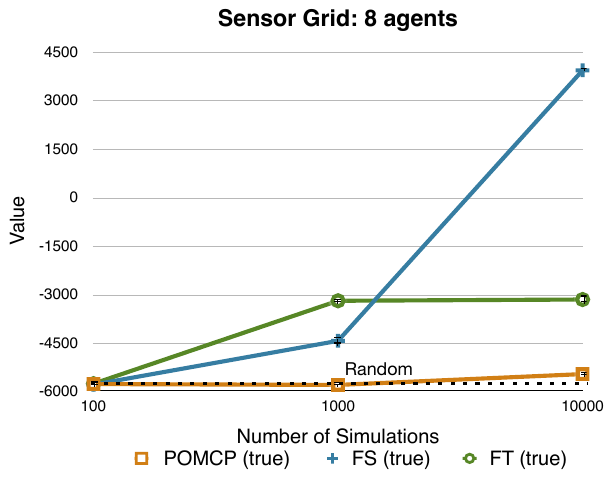} 
}
%%\vspace{-5pt}

%\hspace{-5pt}
\subfigure[BA-MPOMDP results]{
\label{fig:FFGlearn}
\hspace{-5pt}
 \includegraphics[width=2.05in]{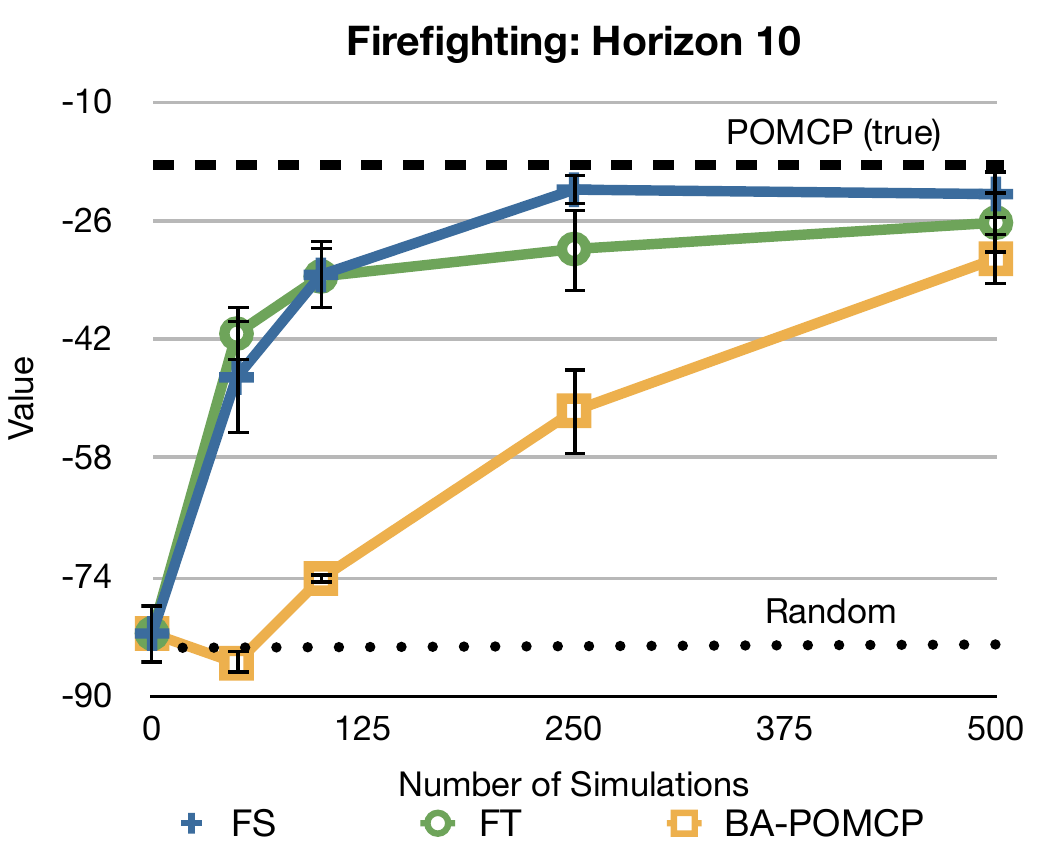}\hspace{-5pt}
\includegraphics[width=2.1in]{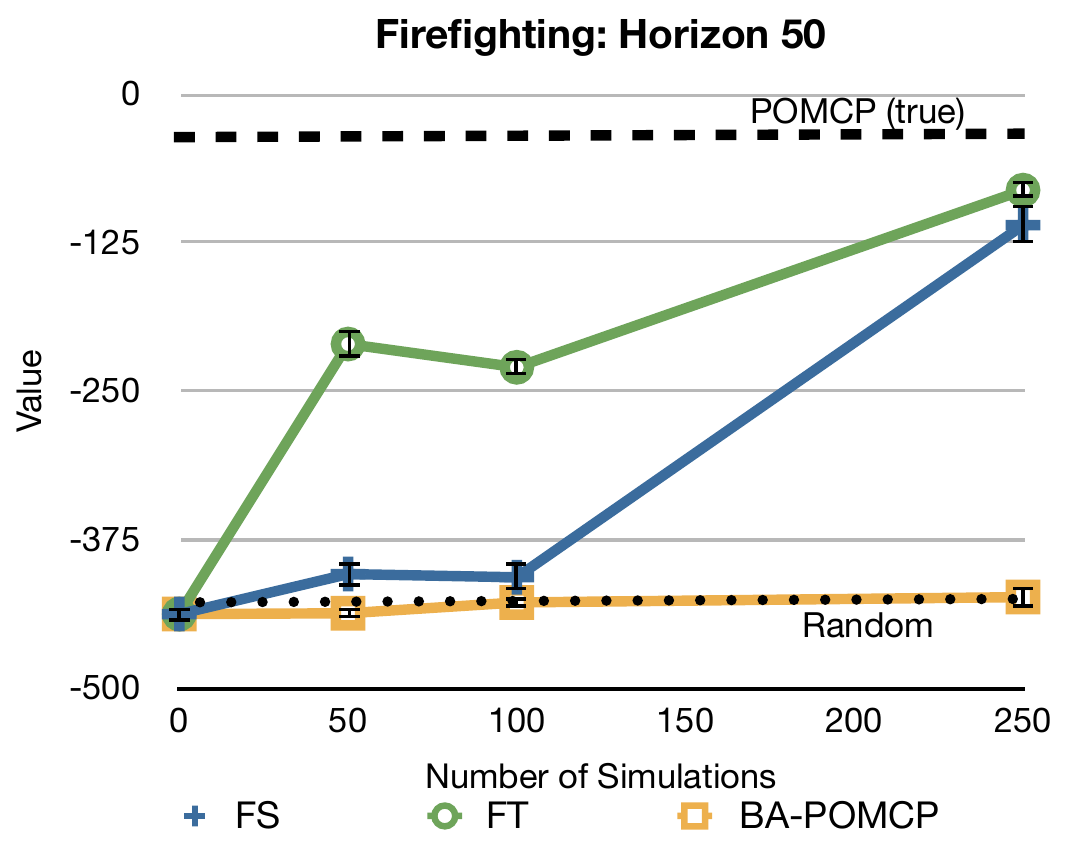}\hspace{-5pt}
\includegraphics[width=2.1in]{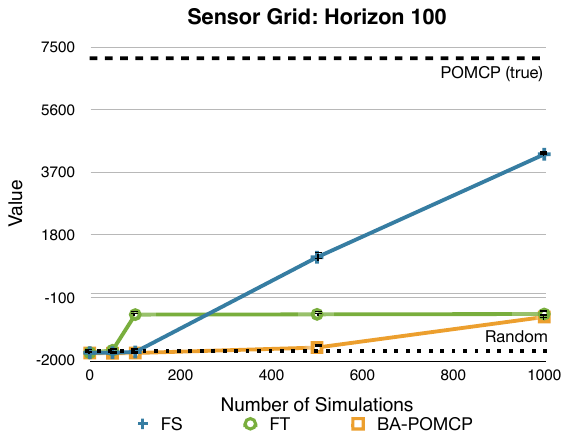} 
}
%%\vspace{-7pt}
\caption{Results for (a) the planning (MPOMDP) case (log scale x-axis) and (b) the learning (BA-MPOMDP) case for the firefighting and sensor grid problems.}

\end{figure}

The results for 4-agent and 10-agent firefighting problems with horizon 10 are shown in Figure
\ref{fig:FFGtrue}.  For the 4-agent problem, POMCP performs poorly with a few simulations, but as the number of simulations increases it outperforms the other methods 
%\todof{I propose to cut this statement perhaps:}
(presumably converging to an optimal solution). 
FT provides a high-quality solution with a very small number of simulations, but the resulting value plateaus due to approximation error. FS also provides a high-quality solution with a very small number of simulations, but is then able to converge to a solution that is near POMCP. In the 10-agent problem, POMCP is only able to generate a solution that is slightly better than random while the FV-POMCP methods are able to
perform much better. In fact, 
FT performs very well with a small number of samples and FS continues to improve until it reaches solution that is similar to FT. 

Similar results are seen in the sensor grid problem. POMCP outperforms a random policy as the number of simulations grows, but FS and FT produce much higher values with the available simulations. FT seems to converge to a low quality solution (in both planning and learning) due to the loss of information about the target's previous position that is no longer known to local factors. In this problem, POMCP requires over 10 minutes for an episode of 10000 simulations, making reducing the number of simulations crucial in problems of this size. 
These results clearly illustrate the benefit of FV-POMCP
by exploiting structure for planning in MASs.

%>>>>analysis

%>>>>>>>>>>>---------------------------------------------------------------------------------
\ourpar{BA-MPOMDPs}
We also investigate the learning setting (i.e., when the agents
are only given the BA-POMDP model).
Here, at the end of each episode, both the state and count vectors are reset to their initial values.
%delta-experimental-setup
Learning in partially observable environments is extremely hard, and there 
may be many equivalence classes of transition and observation models that
are indistinguishable when learning. 
Therefore, %it may be necessary to have an estimate of the transition or observation model. In the following 
we assume a reasonably good model of the transitions (e.g., because the designer may have a good idea of the dynamics), but only a poor estimate of the observation model (because the sensors may be harder to model). 

%baselines
For the BRL setting, we compare to the following baseline methods:
\emph{POMCP}:  regular POMCP applied to the true model using  100,000
simulations (this is the best proxy for, and we expect this to be very close to, the optimal value), and
\emph{BA-POMCP}: regular POMCP applied to the BA-POMDP.

Results for a four agent instance of the fire fighting problem are shown in \fig{FFGlearn}, for
$\hor=10, 50$. 
In both cases, the FS and FT variants approach the
POMCP value.
For a small number of simulations FT learns very quickly, providing
significantly better values than the flat methods and better than  FS
for the increased horizon.  FS learns
more slowly, but the value is better as the number of simulations increases (as seen in the horizon
10 case) due to the use of the full
history.  After more simulations in the
horizon 10 problem, the performance of the flat model (BA-MPOMDP) improves, but
the factored methods still outperform it and this increase is less visible for
the longer horizon problem. 

Similar results are again seen in the four agent sensor grid problem. FT performs the best with a small number of simulations, but as the number increases, FS outperforms other methods. 
Again, for these problems, BA-POMCP requires over 10 minutes for each episode for the largest number of simulations, showing the need for more efficient methods. 
These experiments show that even in challenging multiagent settings
with state uncertainty, BRL methods can learn by
effectively exploiting structure.

%\vspace{-0.15cm}   
\section{Related Work}
MCTS methods have become very popular in games, a type of multiagent setting,
but no action factorization has been exploited so far \cite{Browne12CAAIG}. 
%In particular, 
Progressive widening \cite{Coulom07} and double progressive widening \cite{Couetoux11} have had some success in games with large (or continuous) action spaces.  
%Our method is a more principled method for reducing the action space while also dealing with the large observation space.
The progressive widening methods do not use the structure of the coordination graph in order to generalize value over actions, but instead must find the correct joint action out of the exponentially many that are available (which may require many trajectories).  They are also designed for fully observable scenarios, so they do not address the large observation space in MPOMDPs. %, so particle starvation is still likely to occur. 

The factorization of the history in FTs is not unlike the use of linear function
approximation for the state components in TD-Search \cite{Silver12ML}. However, in
contrast to that method, due to our particular factorization, we can still apply UCT to
aggressively search down the most promising branches of the tree.
While other methods based on Q-learning \cite{Guestrin02ICML,Kok06JMLR} exploit action
factorization, they assume agents observe individual rewards (rather than the global reward that we consider) and it is not clear how these could be incorporated in a UCT-style algorithm.

Locality of interaction has also been considered previously in decentralized POMDP methods \cite{Oliehoek12RLBook,CDC13}
in the form of factored
Dec-POMDPs \cite{Oliehoek13AAMAS,Joni11IJCAI} and networked distributed POMDPs (ND-POMDPs) \cite{Nair05,Akshat09,JillesAAMAS14}. 
These models make strict assumptions about the information that the agents can use to
choose actions (only the past history of individual actions and observations), thereby significantly lowering the resulting value  \cite{Oliehoek08JAIR}. 
ND-POMDPs also impose additional assumptions on the model (transition and observation independence and a factored reward function).
The MPOMDP model, in contrast, does not impose these restrictions. 
%Specifically, in MPOMDPs, 
%In MPOMDPs, 
Instead, in MPOMDPs,
each agent knows the \emph{joint} action-observation history, so there are not
different perspectives by different agents. 
Therefore, 
1) factored Dec-POMDP and ND-POMDP methods do not apply to MPOMDPs; they specify mappings from individual histories to actions (rather than joint histories to joint actions),
2) ND-POMDP methods assume that the value function is \emph{exactly} factored as the sum of local values (`perfect locality of interaction') while in an MPOMDP, the value is only \emph{approximately} factored (since different components can correlate due to conditioning the actions on central information).
While perfect locality of interaction allows a natural factorization of the MPOMDP value function, but our method can be applied to any MPOMDP (i.e., given any factorization of the value function). 
Furthermore, current factored Dec-POMDP and ND-POMDP models generate solutions given the model in an offline fashion, while we consider online methods using a simulator in this paper. 

Our approach builds upon coordination-graphs~\cite{Guestrin01}, to perform the joint action optimization
efficiently, but factorization in one-shot problems has been considered in other settings
too.
Amin et al.~\cite{Amin11UAI} present a method to optimize graphical bandits, which relates to our
optimization approach. Since their approach replaced the UCB functionality, it is not
obvious how their approach could be integrated in POMCP.  Moreover, their work, focuses on
minimizing regret (which is not an issue in our case), and does not apply when the
factorization does not hold.
Oliehoek et al.~\cite{Oliehoek12UAI} present an factored-payoff approach that extends coordination graphs
to imperfect information settings where each agent has its own knowledge. This is not
relevant for our current algorithm, which assumes that joint observations will be received
by a centralized decision maker, but could potentially be useful to relax this assumption.

\section{Conclusions }

We presented the first method to exploit multiagent structure to produce a scalable method for Monte Carlo tree search for POMDPs. 
This approach formalizes a team of agents as a multiagent POMDP, allowing planning and BRL techniques from the POMDP literature to be applied.
However, since the number of joint actions and observations grows exponentially with the number of agents, na\"ive extensions of single agent methods will not scale well. 
To combat this problem, we introduced FV-POMCP, an online planner based on  POMCP \cite{Silver10NIPS23} that exploits multiagent structure using two novel techniques---factored statistics and factored trees--- to 
reduce 1) the number of joint actions and 2) the number of joint histories considered. 
Our empirical results demonstrate that FV-POMCP greatly increases scalability of online planning for MPOMDPs,  solving problems with 10  agents. Further investigation also shows scalability to the much more complex learning problem with four agents.
Our methods could also be used to solve POMDPs and BA-POMDPs with large action and observation spaces
 as well the recent Bayes-Adaptive extension \cite{Ng12} of the self interested I-POMDP model \cite{Gmytrasiewicz05}.

\section*{Acknowledgments}
F.O. is funded by NWO Innovational Research Incentives Scheme Veni \#639.021.336.
C.A was supported by AFOSR MURI project \#FA9550-09-1-0538 and ONR MURI project \#N000141110688.

\bibliographystyle{aaai}
%\bibliography{amsj,BayesBib}
\bibliography{BayesBib}

\appendix

\section{FV-POMCP Pseudo code}
\label{code}

Here we describe in more detail the algorithms for the proposed FV-POMCP variants. Both methods
can be described as a modification of POMCP's \textsc{Simulate} procedure, which is shown in \alg{POMCP}.
Each simulation is started by calling this procedure on the root node (corresponding to the
empty history, or `now') with a state sampled from the current belief. The comments in
\alg{POMCP} should make the code self-explanatory, but for a further explanation we refer
to \cite{Silver10NIPS23}.

\begin{algorithm}
\caption{\label{alg:POMCP}POMCP}
\begin{algorithmic}[1]
%\Input 
%\EndInput
%\Output
%\EndOutput
\Procedure{Simulate}{$s$,$h$,$depth$}
\If {$\gamma^{depth}< \epsilon$} 
\State \Return 0 \Comment{Stop when desired precision reached}
\EndIf
\If {$h \not\in T $} \Comment{if we did not visit this $h$ yet}
\ForAll{$a \in A$}
\State $T(h,a) \gets (N_{init}(h,a),V_{init}(h,a),\emptyset)$ \Comment{initialize counts for
all $a$, and particle filter}
\EndFor
\State \Return Rollout($s$,$h$,$depth$)  \Comment{do a random rollout from this node}
\EndIf
\State $a \gets \arg\max_a' Q(h,a') + c\sqrt{\frac{log N(h)}{N(h,a')}}$ \Comment{max via enumeration}
\State $(s',o,r) \sim \mathcal{G}(s,a)$ \Comment{sample a transition, observation and
reward}
\State $R \gets r + \gamma$Simulate($s'$,$(hao)$,$depth+1$)   \Comment{Recurse and receive
the return $R$}
\State $B(h) \gets B(h) \cup \{s\}$ \Comment{Add $s$ to the particle filter maintained for $h$}
\State $N(h) \gets N(h)+1$          \Comment{Update the number of times $h$ is visited\dots}
\State $N(h,a) \gets N(h,a)+1$        \Comment{\dots and how often we selected action $a$ here}
\State $Q(h,a) \gets Q(h,a) + \frac{R-Q(h,a)}{N(h,a)}$ \Comment{Incremental update of mean return}
\State \Return $R$
\EndProcedure
\end{algorithmic}
\end{algorithm}

The pseudo code for the factored statistics case (FS-FV-POMCP) is shown in \alg{FS}. The comments highlight the important
changes: there is no need to loop over the joint actions for initialization, or for
selecting the maximizing action. Also, the same return $R$ is used to update the active
expert in each component $e$.

\begin{algorithm}
\caption{\label{alg:FS}Factored Statistics}
\begin{algorithmic}[1]
%\Input 
%\EndInput
%\Output
%\EndOutput
\Procedure{Simulate}{$s$,$\vec h$,$depth$}
\If {$\gamma^{depth}< \epsilon$} 
\State \Return 0
\EndIf
\If {$\vec h \not\in T $} 
\ForAll{ $e \in E$ }                    \Comment{initialize statistics for component $e$}
\ForAll{ $ \vec{a}_e \in \jaGS{e}$ }    \Comment{only loop over \emph{local} joint actions}
\State $T(\vec h, \vec{a}_e) \gets (N_{init}(\vec h, \vec{a}_e),V_{init}(\vec h, \vec{a}_e),\emptyset)$
\EndFor
\EndFor
\State \Return Rollout($s$,$\vec h$,$depth$)
\EndIf
%
%\State $\vec a \gets \arg\max_{\vec b} \sum_e Q(\vec h, \vec b) + c\sqrt{\frac{log N(\vec h)}{N(hb)}}$
\State $\vec a \gets \arg\max_{\vec a}' 
\sum_e 
\Big[
Q_e(\vec h, \vec a_e') + c\sqrt{\frac{log N(\vec h+1)}{n_{\vec a_e'}}}$
\Big]
\Comment{via variable elimination}
\State $(s',\vec o,r) \sim \mathcal{G}(s,\vec a)$
\State $R \gets r + \gamma$Simulate($s'$,$(\vec h\vec a\vec o)$,$depth+1$)
\State $B(\vec h) \gets B(\vec h) \cup \{s\}$
\State $N(\vec h) \gets N(\vec h)+1$
\ForAll{$\vec e \in E$}                 \Comment{update the statistics for each component}
\State $n_{\vec a_e} \gets n_{\vec a_e} +1$ 
\State $Q_e(\vec h, \vec a_e) \gets Q_e(\vec h, \vec a_e) + \frac{R-Q_e(\vec h, \vec a_e)}{n_{\vec a_e}}$
\Comment{update the estimation of the expert for $ \vec a_e $}
\EndFor
\State \Return $R$
\EndProcedure
\end{algorithmic}
\end{algorithm}

Finally, \alg{FT} shows the pseudo code for the factored trees variant of FV-POMCP. Like
FSs, FT-FV-POMCP performs the simulations in lockstep. However, as we now maintain
statistics and particle filters for each component $e$ in separate trees, the
initialization and updating of these statistics is slightly different. 
As such, the algorithm makes clear that there is no computational advantage to factored
trees (when compared to FSs), but that the big difference is in the additional
generalization it performs.

The actual planning could take place in a number of ways: one agent could be designated
the planner, which would require this agent to broadcast the computed joint action.
Alternatively, each agent can in parallel perform an identical planning process (by, in
the case of randomized planning, syncing the random number generators). Then each agent
will compute the same joint action and execute its component. 
%Regardless, both methods require sharing of observations: FSs requires full broadcasting
%of individual observation to all other planning agents. FTs could use more local
%communication: 
An interesting direction of
future work is whether the planning itself can be done more effectively by distributing
the task over the agents.

\begin{algorithm}
\caption{\label{alg:FT}Factored Trees}
\begin{algorithmic}[1]
%\Input 
%\EndInput
%\Output
%\EndOutput
\Procedure{Simulate}{$s$,$\vec h$,$depth$}
\If {$\gamma^{depth}< \epsilon$} 
\State \Return 0
\EndIf
\ForAll{ $e \in E$ }                    \Comment{check all components $e$}
    \If {$\vec{h}_e \not\in T_e $}          \Comment{if there is no node for $\vec{h}_e$
    in the tree for component $e$ yet}
        \ForAll{ $ \vec{a}_e \in \jaGS{e}$ }    \Comment{only loop over \emph{local} joint actions}
            \State $T_e(\vec{h}_e, \vec{a}_e) \gets (N_{init}(\vec{h}_e, \vec{a}_e),V_{init}(\vec{h}_e, \vec{a}_e),\emptyset)$
        \EndFor
    \EndIf
\EndFor
\State \Return Rollout($s$,$\vec h$,$depth$)
%
%\State $\vec a \gets \arg\max_{\vec b} \sum_e Q(\vec h, \vec b) + c\sqrt{\frac{log N(\vec h)}{N(hb)}}$
\State $\vec a \gets \arg\max_{\vec a}' 
\sum_e 
\Big[
Q_e(\vec{h}_e, \vec a_e') + c\sqrt{\frac{log N(\vec{h}_e+1)}{n_{\vec a_e'}}}$
\Big]
\Comment{via variable elimination}
\State $(s',\vec o,r) \sim \mathcal{G}(s,\vec a)$
\State $R \gets r + \gamma$Simulate($s'$,$(\vec h\vec a\vec o)$,$depth+1$)
\ForAll{$\vec e \in E$}                 \Comment{update the statisitics for each component}
\State $B(\vec{h}_e) \gets B(\vec{h}_e) \cup \{s\}$  \Comment{update the particle filter of
each tree $e$ with the same $s$}
\State $N(\vec{h}_e) \gets N(\vec{h}_e)+1$
\State $n_{\vec a_e} \gets n_{\vec a_e} +1$ 
\State $Q_e(\vec{h}_e, \vec a_e) \gets Q_e(\vec{h}_e, \vec a_e) + \frac{R-Q_e(\vec{h}_e, \vec a_e)}{n_{\vec a_e}}$
\Comment{update expert for $ \vec h_e, \vec a_e $}
\EndFor
\State \Return $R$
\EndProcedure
\end{algorithmic}
\end{algorithm}

\section{Analysis of Mixture of Experts Optimization}
\label{proofs}

\global\long\def\cip{\overset{p}{\rightarrow}}

Here we analyze the behavior of our mixtures of experts under sample
policy $\jpol$. In performing this analysis we compare to the case
where the true value function $Q(\ja)$ is factored in $E$ components
\[
Q(\ja)=\sum_{e=1}^{E}Q_{e}(\jaG e)
\]
and corrupted by zero-mean noise~$\nu$. As such we establish the
performance in cases where the actual value function is `close'
to factored. In the below, we will write $\neig e$ for the neighborhood
of component $e$. That is the set of other components $e'$ which
have an overlap with $e$: those that have at least one agent participating
in them that also participates in $e$). In this analysis, we will
assume that all experts are weighted uniformly ($\alpha_{e}=\frac{1}{E}$),
and ignore this constant which is not relevant for determining the
maximizing action.

\begin{theorem}The estimate $\hat{Q}$ of $Q$ made by a mixture
of experts converges in probability to the true value plus a sample
policy dependent bias term:

\[
\hat{Q}(\ja)\cip Q(\ja)+B_{\jpol}(\ja).
\]
The bias is given by a sum of biases induced by pairs $e,e'$ of overlapping
value components:
\[
B_{\jpol}(\ja)\defas\sum_{e}\sum_{e'\neq e}\sum_{\overline{\jaG{e'\setminus e}}}\jpol(\overline{\jaG{e'\setminus e}}|\jaG e)\QI{e'}(\overline{\jaG{e'\setminus e}},\jaG{e'\cap e}).
\]
Here $\jaG{e'\cap e}$ is the action of the agents that participate
both in $e$ and $e'$ (specified by $\ja$) and $\overline{\jaG{e'\setminus e}}$
are the actions of agents in $e'$ that are not in $e$ (these are
summed over as emphasised by the overlining).

\end{theorem}

\begin{proof}

Suppose that we have drawn a set of samples $\argsI r1,\dots,\argsI rK$
according to $\jpol$. For each component $e$ and $\jaG e$, we have
an expert that estimates $\hat{Q}(\jaG e)$. Let $\mathcal{R}(\jaG e)$
be the subset of samples received where we took local joint action
$\jaG e$. The corresponding expert will estimate

\[
\hat{Q}(\jaG e):=\frac{1}{\mathcal{R}(\jaG e)}\sum_{r_{\jaG e}\in\mathcal{R}(\jaG e)}r_{\jaG e}
\]
Now, the expected sample that this expert receives is
\begin{eqnarray*}
\E\left[r_{\jaG e}|\jpol\right] & = & \E_{\nu}\left\{ \sum_{\jaG{\excl e}}\jpol(\jaG{\excl e}|\jaG e)Q(\ja)+\nu\right\} \\
& = & \sum_{\jaG{\excl e}}\jpol(\jaG{\excl e}|\jaG e)\left[\sum_{e'}\QI{e'}(\jaG{e'})\right]+\E_{\nu}\nu\\
& = & \QI e(\jaG e)+\sum_{\jaG{\excl e}}\jpol(\jaG{\excl e}|\jaG e)\sum_{e'\neq e}\QI{e'}(\jaG{e'})+0\\
& = & \QI e(\jaG e)+\sum_{e'\neq e}\sum_{\jaG{e'\setminus e}}\jpol(\jaG{e'\setminus e}|\jaG e)\QI{e'}(\jaG{e'\setminus e},\jaG{e'\cap e})
\end{eqnarray*}
which means that the estimate of an expert converges in probability
to a biased estimate
\[
\hat{Q}(\jaG e)\cip\QI e(\jaG e)+\sum_{e'\neq e}\sum_{\jaG{e'\setminus e}}\jpol(\jaG{e'\setminus e}|\jaG e)\QI{e'}(\jaG{e'\setminus e},\jaG{e'\cap e}).
\]
This, in turn means that the mixture of experts

\begin{align*}
\hat{Q}(\ja)=\sum_{e}\argsI{\hat{Q}}e(\jaG e)\cip & \sum_{e}\left[\QI e(\jaG e)+\sum_{e'\neq e}\sum_{\overline{\jaG{e'\setminus e}}}\jpol(\overline{\jaG{e'\setminus e}}|\jaG e)\QI{e'}(\overline{\jaG{e'\setminus e}},\jaG{e'\cap e})\right]\\
= & \sum_{e}\QI e(\jaG e)+\sum_{e}\sum_{e'\neq e}\sum_{\overline{\jaG{e'\setminus e}}}\jpol(\overline{\jaG{e'\setminus e}}|\jaG e)\QI{e'}(\overline{\jaG{e'\setminus e}},\jaG{e'\cap e})\\
= & \sum_{e}\QI e(\jaG e)+B_{\jpol}(\ja),
\end{align*}
as claimed. \end{proof}

As is clear from the definition of the bias term, it is caused by
correlations in the sample policy and the fact that we are over counting
value from other components. When there is no overlap in the payoff
components, and we use a sample policy we use is `\emph{component-wise}',
i.e., $\jpol(\overline{\jaG{e'\setminus e}}|\jaG e)=\jpol(\overline{\jaG{e'\setminus e}}|\jaG e')=\jpol(\overline{\jaG{e'\setminus e}})$,
the effect of this bias can be disregarded: in such a case, even though
all components $e'\neq e$ contribute to the bias of expert $e$ this
bias is constant for those components $e'\not\in\neig e$. Since we
care only about the relative values of joint actions $\ja$, only
overlapping components actually contribute to introducing error. This
is clearly illustrated for the case where there are no overlapping
components.

\begin{theorem}In the case that the value components do not overlap,
mixture of experts optimization recovers the maximizing joint action.
\end{theorem}\begin{proof}

In this case,

\[
\argsI{\hat{Q}}e(\jaG e)\cip\QI e(\jaG e)+\sum_{e'\neq e}\sum_{\jaG{e'}}\jpolG{e'}(\jaG{e'})\QI{e'}(\jaG{e'})
\]
and the bias does not depend on $\jaG e$, such that
\[
\argmax_{\jaG e}\argsI{\hat{Q}}e(\jaG e)\cip\argmax_{\jaG e}\QI e(\jaG e)+C=\argmax_{\jaG e}\QI e(\jaG e).
\]
And for the joint optimization:
\begin{align*}
\hat{Q}(\ja)=\sum_{e}\argsI{\hat{Q}}e(\jaG e)\cip & \sum_{e}\left[\QI e(\jaG e)+\sum_{e'\neq e}\sum_{\jaG{e'}}\jpolG{e'}(\jaG{e'})\QI{e'}(\jaG{e'})\right]\\
= & \sum_{e}\QI e(\jaG e)+\sum_{e}\sum_{e'\neq e}\sum_{\jaG{e'}}\jpolG{e'}(\jaG{e'})\QI{e'}(\jaG{e'})\\
= & \sum_{e}\QI e(\jaG e)+(E-1)\sum_{e}\sum_{\jaG{e}'}\jpolG{e}(\jaG{e}')\QI{e}(\jaG{e}')\\
= & \sum_{e}\left(\QI e(\jaG e)+(E-1)SA(e,\jpol)\right)
\end{align*}
where $SA(e,\jpol)=\sum_{\jaG{e}'}\jpolG{e}(\jaG{e}')\QI{e}(\jaG{e}')$
is the sampled average payoff of component $e$ which affects our
value estimates, but which does not affect the maximizing joint action:
\[
\argmax_{\ja}\left[\sum_{e}\left(\QI e(\jaG e)+(E-1)SA(e,\jpol)\right)\right]=\argmax_{\ja}\sum_{e}\QI e(\jaG e)
\]
\end{proof}

Similar reasoning can be used to establish bounds on the performance
of mixture of expert optimization in cases with overlap, as is shown
by the next theorem.

\begin{theorem} If for all overlapping components $e,e'$, and any
two `intersection action profiles' $\jaG{e'\cap e},\jaG{e'\cap e}'$
for their intersection, the true value function satisfies that

\[
\forall\jaG{e'\setminus e}\qquad
\QI{e'}(\jaG{e'\setminus e},\jaG{e'\cap e})
-
\QI{e'}(\jaG{e'\setminus e},\jaG{e'\cap e}')
\leq\frac{\epsilon}{E\cdot\left|\neig e\right|\cdot\left|\jaGS{e'\setminus e}\right|\cdot\jpol(\jaG{e'\setminus e})},
\]
with $\left|\jaGS{e'\setminus e}\right|$ the number of intersection
action profiles, then mixture of experts optimization, in the limit
will return a joint action whose value solution that lies within $\epsilon$
of the optimal solution. \end{theorem}

\begin{proof}

Bias by itself is no problem, but different bias for different joint
actions is, because that may cause us to select the wrong action.
As such, we set out to bound
\[
\forall_{\ja,\ja'}\qquad\left|B_{\jpol}(\ja)-B_{\jpol}(\ja')\right|\leq\epsilon.
\]
As explained, only terms where the bias is different for two actions
$\ja,\ja'$ matter. As such we will omit terms that cancel out. In
particular, we have that if two components $e,e'$ do not overlap,
the expression

\[
\sum_{\overline{\jaG{e'\setminus e}}}\jpol(\overline{\jaG{e'\setminus e}}|\jaG e)\QI{e'}(\overline{\jaG{e'\setminus e}},\jaG{e'\cap e})-\sum_{\overline{\jaG{e'\setminus e}}}\jpol(\overline{\jaG{e'\setminus e}}|\jaG e')\QI{e'}(\overline{\jaG{e'\setminus e}},\jaG{e'\cap e}')
\]
reduces to $\sum_{\overline{\jaG{e'}}}\jpol(\overline{\jaG{e'}}|\jaG e)\QI{e'}(\overline{\jaG{e'}})-\sum_{\overline{\jaG{e'}}}\jpol(\overline{\jaG{e'}}|\jaG e')\QI{e'}(\overline{\jaG{e'}})$
and hence vanishes under a componentwise policy. We use this insight
to define the bias in terms of neighborhood bias. Let us write $\neig{e}$
for the set of edges $\left\{ e'\right\} $ that have overlap with
$e$, and let's write $\jaG{\neig{e}}$ for the joint action that
specifies actions for that entire overlap, then we can write this
as:

\[
B_{\jpol}(\ja)\defas\sum_{e=1}^{E}B_{\jpol}^{e}(\jaG{\neig{e}})
\]
with\textbf{
\[
B_{\jpol}^{e}(\jaG{\neig{e}})=\sum_{e'\in\neig e}B_{\jpol}^{{e'\cap e}}(\jaG{e'\cap e})
\]
}component $e$'s `neighborhood bias', with

\[
B_{\jpol}^{{e'\rightarrow e}}(\jaG{e})\triangleq\sum_{\overline{\jaG{e'\setminus e}}}\jpol(\overline{\jaG{e'\setminus e}}|\jaG e)\QI{e'}(\overline{\jaG{e'\setminus e}},\jaG{e'\cap e})
\]
the `intersection bias': the bias introduced on $e$ via the overlap
between $e'$ and $e$.

Now, we show how we can guarantee that the bias is small, by guaranteeing
that the intersection biases are small. This gives a conservative
bound, since different biases may very well cancel out. Artbitrarily
select two actions, W.l.o.g. we select $\ja$ to be the larger-biased
joint action. We need to guarantee that

\begin{align*}
& B_{\jpol}(\ja)-B_{\jpol}(\ja')\leq\epsilon\\
\leftrightarrow\quad & \sum_{e=1}^{E}\sum_{e'\in\neig e}B_{\jpol}^{{e'\rightarrow e}}(\jaG{e})-\sum_{e=1}^{E}\sum_{e'\in\neig e}B_{\jpol}^{{e'\rightarrow e}}(\jaG{e}')\leq\epsilon\\
\leftrightarrow\quad & \sum_{e=1}^{E}\left[\sum_{e'\in\neig e}B_{\jpol}^{{e'\rightarrow e}}(\jaG{e})-\sum_{e'\in\neig e}B_{\jpol}^{{e'\rightarrow e}}(\jaG{e}')\right]\leq\epsilon\\
\leftarrow\quad & \sum_{e'\in\neig e}B_{\jpol}^{{e'\rightarrow e}}(\jaG{e})-\sum_{e'\in\neig e}B_{\jpol}^{{e'\rightarrow e}}(\jaG{e}')\leq\frac{\epsilon}{E}\quad\forall e\\
\leftrightarrow\quad & \sum_{e'\in\neig e}\left[B_{\jpol}^{{e'\rightarrow e}}(\jaG{e})-B_{\jpol}^{{e'\rightarrow e}}(\jaG{e}')\right]\leq\frac{\epsilon}{E}\quad\forall e\\
\leftarrow\quad & \left[B_{\jpol}^{{e'\rightarrow e}}(\jaG{e})-B_{\jpol}^{{e'\rightarrow e}}(\jaG{e}')\right]\leq\frac{\epsilon}{E\times\left|\neig e\right|}\quad\forall e,e'
\end{align*}
let $\epsilon'\triangleq\frac{\epsilon}{E\times\left|\neig e\right|}$,
we want
\begin{eqnarray*}
& & \left[B_{\jpol}^{{e'\rightarrow e}}(\jaG{e})-B_{\jpol}^{{e'\rightarrow e}}(\jaG{e}')\right]\leq\epsilon'\\
& & \sum_{\overline{\jaG{e'\setminus e}}}\jpol(\overline{\jaG{e'\setminus e}}|\jaG e)\QI{e'}(\overline{\jaG{e'\setminus e}},\jaG{e'\cap e})-\sum_{\overline{\jaG{e'\setminus e}}}\jpol(\overline{\jaG{e'\setminus e}}|\jaG e')\QI{e'}(\overline{\jaG{e'\setminus e}},\jaG{e'\cap e}')\leq\epsilon'\\
\leftrightarrow & & \sum_{\overline{\jaG{e'\setminus e}}}\left[\jpol(\overline{\jaG{e'\setminus e}}|\jaG e)\QI{e'}(\overline{\jaG{e'\setminus e}},\jaG{e'\cap e})-\jpol(\overline{\jaG{e'\setminus e}}|\jaG e')\QI{e'}(\overline{\jaG{e'\setminus e}},\jaG{e'\cap e}')\right]\leq\epsilon'\\
\leftarrow & & \jpol(\overline{\jaG{e'\setminus e}}|\jaG e)\QI{e'}(\overline{\jaG{e'\setminus e}},\jaG{e'\cap e})-\jpol(\overline{\jaG{e'\setminus e}}|\jaG e')\QI{e'}(\overline{\jaG{e'\setminus e}},\jaG{e'\cap e}')\leq\frac{\epsilon'}{\left|\jaGS{e'\setminus e}\right|}\qquad\forall_{\overline{\jaG{e'\setminus e}}}
\end{eqnarray*}
Under a component-wise policy $\jpol(\overline{\jaG{e'\setminus e}}|\jaG e)=\jpol(\overline{\jaG{e'\setminus e}}|\jaG e')=\jpol(\overline{\jaG{e'\setminus e}})$
and thus
\begin{eqnarray*}
& & \jpol(\overline{\jaG{e'\setminus e}})\left[\QI{e'}(\overline{\jaG{e'\setminus e}},\jaG{e'\cap e})-\QI{e'}(\overline{\jaG{e'\setminus e}},\jaG{e'\cap e}')\right]\leq\frac{\epsilon'}{\left|\jaGS{e'\setminus e}\right|}\qquad\forall_{\overline{\jaG{e'\setminus e}}}\\
\leftrightarrow & & \QI{e'}(\overline{\jaG{e'\setminus e}},\jaG{e'\cap e})-\QI{e'}(\overline{\jaG{e'\setminus e}},\jaG{e'\cap e}')\leq\frac{\epsilon'}{\left|\jaGS{e'\setminus e}\right|\jpol(\overline{\jaG{e'\setminus e}})}\qquad\forall_{\overline{\jaG{e'\setminus e}}}
\end{eqnarray*}
Putting it all together, if the factorized Q function satisfies that,
for all the component intersections, the following holds:
\[
\forall\overline{\jaG{e'\setminus e}},\jaG{e'\cap e},\jaG{e'\cap e}'\qquad\QI{e'}(\overline{\jaG{e'\setminus e}},\jaG{e'\cap e})-\QI{e'}(\overline{\jaG{e'\setminus e}},\jaG{e'\cap e}')\leq\frac{\epsilon}{E\left|\neig e\right|\left|\jaGS{e'\setminus e}\right|\jpol(\overline{\jaG{e'\setminus e}})}
\]
we are guaranteed to find an $\epsilon$-(absolute-error)-approximate
solution. \end{proof}

\end{document}